\documentclass[twoside]{article}

\usepackage[preprint]{aistats2026}

\usepackage{comment}
\usepackage{graphicx}
\usepackage{svg}
\usepackage{tikz-cd}
\usepackage{pgfkeys}
\usepackage{todonotes}
\usetikzlibrary{shapes.geometric}
\usetikzlibrary{decorations.markings}

\usepackage{tikzit}

\tikzstyle{morphism}=[fill=white, draw=black, shape=rectangle]
\tikzstyle{medium box}=[fill=white, draw=black, shape=rectangle, minimum width=1.3cm, minimum height=0.7cm]
\tikzstyle{large morphism}=[fill=white, draw=black, shape=rectangle, minimum width=1.7cm]
\tikzstyle{bn}=[fill=black, draw=black, shape=circle, inner sep=1.5pt]
\tikzstyle{state}=[fill=white, draw=black, regular polygon, regular polygon sides=3, minimum width=0.8cm, shape border rotate=180, inner sep=0pt]
\tikzstyle{medium state}=[fill=white, draw=black, regular polygon, regular polygon sides=3, minimum width=1.3cm, inner sep=0pt, shape border rotate=180]
\tikzstyle{large state}=[fill=white, draw=black, regular polygon, regular polygon sides=3, minimum width=2.2cm, shape border rotate=180, inner sep=0pt]
\tikzstyle{wide state}=[fill=white, draw=black, shape=isosceles triangle, minimum width=0.8cm, shape border rotate=270, inner sep=1.4pt, minimum height=0.5cm, isosceles triangle apex angle=80]
\tikzstyle{wn}=[fill=white, draw=black, shape=circle, inner sep=1.5pt]
\tikzstyle{blue morphism}=[fill=white, draw={rgb,255: red,15; green,0; blue,150}, shape=rectangle, text={rgb,255: red,15; green,0; blue,150}, tikzit category=blue]
\tikzstyle{blue state}=[fill=white, draw={rgb,255: red,15; green,0; blue,150}, shape=circle, regular polygon, regular polygon sides=3, minimum width=0.8cm, shape border rotate=180, inner sep=0pt, text={rgb,255: red,15; green,0; blue,150}, tikzit category=blue]
\tikzstyle{blue node}=[fill={rgb,255: red,15; green,0; blue,150}, draw={rgb,255: red,15; green,0; blue,150}, shape=circle, tikzit category=blue, inner sep=1.5pt]
\tikzstyle{blue}=[text={rgb,255: red,15; green,0; blue,150}, tikzit draw={rgb,255: red,191; green,191; blue,191}, tikzit category=blue, tikzit fill=white, inner sep=0mm]
\tikzstyle{blue wide state}=[fill=white, draw={rgb,255: red,15; green,0; blue,150}, text={rgb,255: red,15; green,0; blue,150}, shape=isosceles triangle, minimum width=0.8cm, shape border rotate=270, inner sep=1.4pt, minimum height=0.5cm, isosceles triangle apex angle=80]
\tikzstyle{red node}=[fill={rgb,255: red,150; green,0; blue,2}, draw={rgb,255: red,150; green,0; blue,2}, shape=circle, inner sep=1.5pt]
\tikzstyle{Purple node}=[fill={rgb,255: red,150; green,0; blue,150}, draw={rgb,255: red,150; green,0; blue,150}, shape=circle, inner sep=1.5pt]
\tikzstyle{red}=[text={rgb,255: red,150; green,0; blue,2}, inner sep=0mm, tikzit fill=white, tikzit draw={rgb,255: red,191; green,191; blue,191}]
\tikzstyle{purple}=[text={rgb,255: red,150; green,0; blue,150}, inner sep=0mm, tikzit fill=white, tikzit draw={rgb,255: red,191; green,191; blue,191}]
\tikzstyle{white morphism}=[fill=white, draw=white, shape=rectangle, tikzit draw={rgb,255: red,139; green,139; blue,139}]
\tikzstyle{curly brace}=[decorate, decoration={brace,amplitude=5pt}]
\tikzstyle{costate}=[fill=white, draw=black, shape=circle, regular polygon, regular polygon sides=3, minimum width=0.8cm, inner sep=0pt]

\tikzstyle{arrow}=[->]
\tikzstyle{dashed box}=[-, dashed]
\tikzstyle{blue arrow}=[-, draw={rgb,255: red,15; green,0; blue,150}, tikzit category=blue]
\tikzstyle{mapsto}=[{|->}]
\tikzstyle{double wire}=[-, double]

\usepackage{amsmath}
\usepackage{amssymb, amsfonts, amsthm, bm}
\usepackage[capitalize]{cleveref}

\tikzset{
  myblock/.style={
    draw,text width=0.6cm,minimum height=0.4cm,align=center
  },
  longblock/.style={
    draw,text width=1.6cm,minimum height=0.4cm,align=center
  },
  arrowleft/.style 2 args={
    decoration={
      markings,
      mark=at position #1 with {\node[left] {#2};}, 
    },
  postaction=decorate  
  },
  arrowright/.style 2 args={
    decoration={
      markings,
      mark=at position #1 with {\node[right] {#2};}, 
    },
  postaction=decorate  
  },
  triangle/.style = {regular polygon, regular polygon sides=3,
    draw, text width=0.6cm,minimum height=0.4cm,
    inner sep=-2mm, outer sep=0mm,
    align=center,
    shape border rotate=-180}
}

\makeatletter
\def\thm@space@setup{%
  \thm@preskip=10pt 
  \thm@postskip=10pt 
}
\makeatother

\newtheorem{theorem}{Theorem}[section]
\newtheorem{proposition}{Proposition}[section]
\newtheorem{corollary}{Corollary}[theorem]
\newtheorem{lemma}{Lemma}[section]

\theoremstyle{definition}
\newtheorem{definition}{Definition}[section]

\theoremstyle{definition}

\theoremstyle{definition}
\newtheorem{example}{Example}[section]

\theoremstyle{remark}
\newtheorem*{remark}{Remark}
\newcommand{\newterm}[1]{\textbf{#1}}

\usepackage[round]{natbib}

\begin{document}

%

%

\twocolumn[

\aistatstitle{Causal Abstractions, Categorically Unified}

\aistatsauthor{Markus Englberger \And Devendra Singh Dhami}

\aistatsaddress{Department of Mathematics and Computer Science, \\ Eindhoven University of Technology}]

\begin{abstract}
We present a categorical framework for relating causal models that represent the same system at different levels of abstraction. We define a causal abstraction as natural transformations between appropriate Markov functors, which concisely consolidate desirable properties a causal abstraction should exhibit. Our approach unifies and generalizes previously considered causal abstractions, and we obtain categorical proofs and generalizations of existing results on causal abstractions. Using string diagrammatical tools, we can explicitly describe the graphs that serve as consistent abstractions of a low-level graph under interventions. We discuss how methods from mechanistic interpretability, such as circuit analysis and sparse autoencoders, fit within our categorical framework. We also show how applying do-calculus on a high-level graphical abstraction of an acyclic-directed mixed graph (ADMG), when unobserved confounders are present, gives valid results on the low-level graph, thus generalizing an earlier statement by \citet{CDAGS}. We argue that our framework is more suitable for modeling causal abstractions compared to existing categorical frameworks. Finally, we discuss how notions such as $\tau$-consistency and constructive $\tau$-abstractions can be recovered with our framework. 

\end{abstract}

\section{INTRODUCTION}

This paper presents a unified categorical framework for causal abstractions, synthesizing and extending previous work by \cite{rubenstein2017causalconsistencystructuralequation, beckers2019abstractingcausalmodels, CDAGS, otsuka2022}. By defining causal abstractions as natural transformations involving general Markov categories, we obtain a general treatment of deterministic and probabilistic as well as discrete, continuous, or mixed random variables. Our framework also models abstractions where there isn't a simple one-to-one mapping between interventions on high-level and low-level variables. We achieve this by relaxing the assumption of a strict monoidal functor to a lax monoidal functor. Further, by defining an alternative causal abstraction with a reversed natural transformation, we differentiate between two distinct types of abstraction. One type clusters variable domains based on their shared effect on causal children, while the other clusters them based on how they are affected by causal parents.

We relate our framework to earlier work on causal abstractions. In \cite{CDAGS}, the authors show how Causal Bayesian Networks with unobserved confounders can be abstracted by partitioning variables, such that interventional distribution also factorizes over the clustered graph and such that applying do-calculus on the high-level clustered graph produces valid results for the low-level graph. We generalize these results employing concise categorical proofs. 
In \cite{beckers2019abstractingcausalmodels}, the authors introduce strong $\tau$-abstractions and a stronger version called constructive $\tau$-abstractions where there has to be an alignment between high-level variables and subsets of low-level variables. They conjecture that under a few minor technical conditions, every strong $\tau$-construction is also a constructive $\tau$-abstraction. By pointing to our earlier discussion of relaxing the assumption of strict to lax monoidal functors, we can describe examples of strong $\tau$-abstractions that are not constructive $\tau$-abstractions.

Causal abstractions have also been introduced in the field of mechanistic interpretability, see e.g. \cite{geiger2025causalabstractiontheoreticalfoundation}. In light of the linear representation hypothesis and the phenomenon of superposition, the concepts one would like to be able to intervene do not generally coincide with individual or sets of neurons. We can again model this via lax monoidal Markov functors and frame the task of training an appropriate sparse autoencoder - where the concepts are aligned - as finding an appropriate natural transformation between a lax Markov functor and a strict Markov functor.
The closest framework to ours is the work by \cite{otsuka2022}. We argue that our framework is comparatively more suitable as an abstract framework for causal abstractions.

\section{MARKOV CATEGORIES AND CAUSAL MODELS}
In this section, we introduce a categorical formulation of causal models. 
\cite{Fritz_2023} introduced \textit{Markov categories}, representing the morphisms in a monoidal category graphically as string diagrams:
\begin{definition}
	\label{cd_cat}
	A \newterm{Markov category} is a symmetric monoidal category $(M, \otimes, I)$ with a
	commutative comonoid structure on each object $X$, consisting of a comultiplication and counit, called \textbf{copying} and \textbf{discarding}:
	\begin{equation*}
			\begin{tikzpicture}
	\begin{pgfonlayer}{nodelayer}
		\node [style=none] (0) at (0, 0.5) {};
		\node [style=bn] (1) at (0, 0.75) {};
		\node [style=none] (2) at (-0.5, 1.25) {};
		\node [style=none] (3) at (0.5, 1.25) {};
		\node [style=none] (4) at (0, 0.25) {$X$};
		\node [style=none] (5) at (-2.5, 0.75) {$\mathsf{copy}_X$};
		\node [style=none] (6) at (-1.25, 0.75) {=};
	\end{pgfonlayer}
	\begin{pgfonlayer}{edgelayer}
		\draw (0.center) to (1);
		\draw [bend left=45, looseness=1.25] (1) to (2.center);
		\draw [bend right=45, looseness=1.25] (1) to (3.center);
	\end{pgfonlayer}
\end{tikzpicture} \hspace{1cm}
			\begin{tikzpicture}
	\begin{pgfonlayer}{nodelayer}
		\node [style=none] (5) at (-1, 0.25) {};
		\node [style=bn] (6) at (-1, 1) {};
		\node [style=none] (7) at (-1, 0) {$X$};
		\node [style=none] (8) at (-3.25, 0.5) {$\mathsf{discard}_X$};
		\node [style=none] (9) at (-2, 0.5) {=};
	\end{pgfonlayer}
	\begin{pgfonlayer}{edgelayer}
		\draw (5.center) to (6);
	\end{pgfonlayer}
\end{tikzpicture}
	\end{equation*}
	satisfying the commutative comonoid equations,
	\begin{equation*}\begin{split}\label{comonoid_eq}
			\begin{tikzpicture}[scale=0.6]
	\begin{pgfonlayer}{nodelayer}
		\node [style=none] (0) at (-9, 1.25) {};
		\node [style=none] (1) at (-8, 1.25) {};
		\node [style=none] (2) at (-9, 0.25) {};
		\node [style=none] (3) at (-8, 0.25) {};
		\node [style=none] (5) at (-9, 0) {};
		\node [style=none] (6) at (-8, 0) {};
		\node [style=none] (8) at (-8.5, -1) {};
		\node [style=bn] (9) at (-8.5, -0.5) {};
		\node [style=none] (10) at (-7, 0) {$=$};
		\node [style=none] (11) at (-6, 1) {};
		\node [style=none] (12) at (-5, 1) {};
		\node [style=none] (14) at (-5.5, -1) {};
		\node [style=none] (15) at (-8.5, 0.75) {};
		\node [style=none] (16) at (-8.5, -1.25) {$X$};
		\node [style=none] (17) at (1, -1.25) {$X$};
		\node [style=bn] (18) at (-5.5, 0) {};
		\node [style=none] (19) at (-5.5, -1.25) {$X$};
		\node [style=none] (20) at (-1.5, -1.25) {$X$};
		\node [style=none] (22) at (-1, 1) {};
		\node [style=none] (23) at (-1.5, -1) {};
		\node [style=bn] (24) at (-1.5, 0) {};
		\node [style=bn] (27) at (-2, 1) {};
		\node [style=none] (28) at (0, 0) {$=$};
		\node [style=none] (29) at (1, 1) {};
		\node [style=none] (30) at (1, -1) {};
		\node [style=none] (32) at (4, 1) {};
		\node [style=none] (33) at (4.75, 1) {};
	\end{pgfonlayer}
	\begin{pgfonlayer}{edgelayer}
		\draw (2.center) to (5.center);
		\draw (3.center) to (6.center);
		\draw [bend left] (6.center) to (9);
		\draw [bend right, looseness=0.75] (5.center) to (9);
		\draw (9) to (8.center);
		\draw [bend left, looseness=0.75] (2.center) to (15.center);
		\draw [bend right, looseness=0.75] (3.center) to (15.center);
		\draw [bend right] (15.center) to (1.center);
		\draw [bend left] (15.center) to (0.center);
		\draw (14.center) to (18);
		\draw [bend left] (18) to (11.center);
		\draw [bend right] (18) to (12.center);
		\draw (23.center) to (24);
		\draw [bend right] (24) to (22.center);
		\draw [bend right] (27) to (24);
		\draw (30.center) to (29.center);
	\end{pgfonlayer}
\end{tikzpicture}

	\end{split}\end{equation*}

        \begin{equation*}\begin{split}
            \begin{tikzpicture}[scale=0.6]
	\begin{pgfonlayer}{nodelayer}
            \node [style=none] (26) at (4.75, -1.25) {$X$};
            \node [style=none] (31) at (9, -1.25) {$X$};
		\node [style=none] (34) at (5.75, 1) {};
		\node [style=bn] (35) at (5.25, 0.25) {};
		\node [style=bn] (36) at (4.75, -0.25) {};
		\node [style=none] (37) at (4.75, -1) {};
		\node [style=none] (38) at (6.75, 0) {$=$};
		\node [style=none] (39) at (9.75, 1) {};
		\node [style=none] (40) at (9, 1) {};
		\node [style=none] (41) at (8, 1) {};
		\node [style=bn] (42) at (8.5, 0.25) {};
		\node [style=bn] (43) at (9, -0.25) {};
		\node [style=none] (44) at (9, -1) {};
	\end{pgfonlayer}
	\begin{pgfonlayer}{edgelayer}
		\draw [bend right, looseness=0.75] (33.center) to (35);
		\draw [bend left, looseness=0.75] (34.center) to (35);
		\draw [bend right] (32.center) to (36);
		\draw [bend right=15] (36) to (35);
		\draw (37.center) to (36);
		\draw [bend left, looseness=0.75] (40.center) to (42);
		\draw [bend right, looseness=0.75] (41.center) to (42);
		\draw [bend left] (39.center) to (43);
		\draw [bend left=15] (43) to (42);
		\draw (44.center) to (43);
	\end{pgfonlayer}
\end{tikzpicture}
        \end{split}
        \end{equation*}
   
	The comonoid structures must be multiplicative with respect to the monoidal structure:
	\begin{align*}\begin{split}\label{compatible_monoid}
			\begin{tikzpicture}[scale=0.6]
	\begin{pgfonlayer}{nodelayer}
		\node [style=none] (46) at (-6.25, -0.5) {};
		\node [style=bn] (47) at (-6.25, 0.25) {};
		\node [style=none] (48) at (-7, 1) {};
		\node [style=none] (49) at (-5.5, 1) {};
		\node [style=none] (50) at (-4.5, 0) {=};
		\node [style=none] (51) at (-2.75, -0.5) {};
		\node [style=bn] (52) at (-2.75, 0.25) {};
		\node [style=none] (53) at (-3.5, 1) {};
		\node [style=none] (54) at (-2, 1) {};
		\node [style=none] (55) at (-2.25, -0.5) {};
		\node [style=bn] (56) at (-2.25, 0.25) {};
		\node [style=none] (57) at (-3, 1) {};
		\node [style=none] (58) at (-1.5, 1) {};
		\node [style=none] (72) at (-6.25, -0.75) {$X \otimes Y$};
		\node [style=none] (73) at (-2.75, -0.75) {$X$};
		\node [style=none] (74) at (-2.25, -0.75) {$Y$};
            \node [style=none] (0) at (0.75, -0.5) {};
		\node [style=none] (1) at (2.5, 0) {=};
		\node [style=none] (2) at (4, -0.5) {};
		\node [style=none] (3) at (5, -0.5) {};
		\node [style=bn] (6) at (0.75, 0.5) {};
		\node [style=bn] (7) at (4, 0.5) {};
		\node [style=bn] (8) at (5, 0.5) {};
		\node [style=none] (14) at (0.75, -0.75) {$X\otimes Y$};
		\node [style=none] (15) at (4, -0.75) {$X$};
		\node [style=none] (16) at (5, -0.75) {$Y$};
	\end{pgfonlayer}
	\begin{pgfonlayer}{edgelayer}
		\draw [bend left=45, looseness=0.75] (47) to (48.center);
		\draw [bend right=45, looseness=0.75] (47) to (49.center);
		\draw (47) to (46.center);
		\draw [bend left=45, looseness=0.75] (52) to (53.center);
		\draw [bend right=45, looseness=0.75] (52) to (54.center);
		\draw (52) to (51.center);
		\draw [bend left=45, looseness=0.75] (56) to (57.center);
		\draw [bend right=45, looseness=0.75] (56) to (58.center);
		\draw (56) to (55.center);
		\draw (6) to (0.center);
		\draw (7) to (2.center);
		\draw (8) to (3.center);
	\end{pgfonlayer}
\end{tikzpicture} \\
            \begin{tikzpicture}[scale=0.6]
	\begin{pgfonlayer}{nodelayer}
		\node [style=none] (4) at (2.25, -0.5) {};
		\node [style=none] (5) at (4, 0) {=};
		\node [style=bn] (9) at (2.25, 0.5) {};
		\node [style=none] (10) at (5, 0.75) {};
		\node [style=none] (11) at (5, -0.75) {};
		\node [style=none] (12) at (6.25, -0.75) {};
		\node [style=none] (13) at (6.25, 0.75) {};
		\node [style=none] (17) at (2.25, -0.75) {$I$};
            \node [style=none] (59) at (9.25, -0.5) {};
		\node [style=bn] (60) at (9.25, 0.25) {};
		\node [style=none] (61) at (8.5, 1) {};
		\node [style=none] (62) at (10, 1) {};
		\node [style=none] (63) at (12, 0.75) {};
		\node [style=none] (64) at (12, -0.75) {};
		\node [style=none] (65) at (13.25, -0.75) {};
		\node [style=none] (66) at (13.25, 0.75) {};
		\node [style=none] (67) at (11, 0) {=};
		\node [style=none] (75) at (9.25, -0.75) {$I$};
	\end{pgfonlayer}
	\begin{pgfonlayer}{edgelayer}
		\draw (9) to (4.center);
		\draw [style=dashed box] (10.center) to (11.center);
		\draw [style=dashed box] (11.center) to (12.center);
		\draw [style=dashed box] (12.center) to (13.center);
		\draw [style=dashed box] (13.center) to (10.center);
            \draw [bend left=45, looseness=0.75] (60) to (61.center);
		\draw [bend right=45, looseness=0.75] (60) to (62.center);
		\draw (60) to (59.center);
		\draw [style=dashed box] (63.center) to (64.center);
		\draw [style=dashed box] (64.center) to (65.center);
		\draw [style=dashed box] (65.center) to (66.center);
		\draw [style=dashed box] (66.center) to (63.center);
	\end{pgfonlayer}
\end{tikzpicture} 
	\end{split}\end{align*}
    The monoidal unit $I$ is required to be terminal.
     A \textbf{Markov functor} is a strict monoidal functor between Markov categories respecting the Markov structure.
\end{definition}

Relevant Markov categories include the category of stochastic Markov kernels \textbf{Stoch} and the category of sets \textbf{Set}. For an extensive introduction and list of Markov categories, we refer to \cite{Fritz_2020}. For an introduction to monoidal categories, we refer to \cite{monoidal}.

To define our causal abstraction  framework, we also need the notion of \textbf{deterministic} morphisms in a Markov category  (\cite{CARBONI198711}):
\begin{definition}
	\label{defn_det}
	A morphism $p : X \to Y$ in a Markov category is \textbf{deterministic} if it respects the comultiplication,

		\begin{tikzpicture}[scale=0.5]
	\begin{pgfonlayer}{nodelayer}
        \node (31) at (-6,0) {};
		\node [style=none] (13) at (-3.25, -2) {};
		\node [style=bn] (14) at (-3.25, 0) {};
		\node [style=none] (15) at (-4.25, 1.25) {};
		\node [style=none] (16) at (-2.25, 1.25) {};
		\node [style=none] (17) at (-2.25, 1.25) {};
		\node [style=morphism] (18) at (-4.25, 1.25) {$p$};
		\node [style=morphism] (19) at (-2.25, 1.25) {$p$};
		\node [style=none] (20) at (3, -2) {};
		\node [style=bn] (21) at (3, 0) {};
		\node [style=none] (22) at (2, 1) {};
		\node [style=none] (23) at (4, 1) {};
		\node [style=none] (24) at (4, 1) {};
		\node [style=none] (25) at (-4.25, 2.25) {};
		\node [style=none] (26) at (-2.25, 2.25) {};
		\node [style=none] (27) at (2, 2.25) {};
		\node [style=none] (28) at (4, 2.25) {};
		\node [style=none] (29) at (0, 0) {$=$};
		\node [style=morphism] (30) at (3, -1) {$p$};
	\end{pgfonlayer}
	\begin{pgfonlayer}{edgelayer}
		\draw [style=none] (13.center) to (14);
		\draw [style=none, bend left=45] (14) to (15.center);
		\draw [style=none, bend right=45] (14) to (16.center);
		\draw [style=none] (20.center) to (21);
		\draw [style=none, bend left=45] (21) to (22.center);
		\draw [style=none, bend right=45] (21) to (23.center);
		\draw (25.center) to (18);
		\draw (26.center) to (19);
		\draw (27.center) to (22.center);
		\draw (28.center) to (24.center);
		\draw (15.center) to (18);
		\draw (17.center) to (19);
	\end{pgfonlayer}
\end{tikzpicture}
	
\end{definition}
In Stoch, the conventionally known concept of deterministic morphisms and categorical notion coincide.

Furthermore, every Directed Acyclic Graph (DAG) has an associated Markov category~(\cite{Fritz_2020}):
\begin{definition}\label{def:generatedCats}
    Given a DAG $L=(\mathbf{V^L},E^L)$, let \textbf{Free$_L$} be the Markov category freely generated by the nodes $\mathbf{V^L}$ as objects and boxes \begin{tikzpicture}[baseline,
    box/.style={draw,rounded corners,minimum width=8mm,minimum height=5mm,inner sep=2pt}]
  \scriptsize
  \node[box] (F) {};
  \draw[arrowright={0.5}{$A$}]
    ([yshift=7pt]F.north) -- (F.north);
  \draw[arrowleft={0.5}{$pa^L_A$}]
    ([xshift=-7pt, yshift=-7pt]F.south) -- ([xshift=-7pt]F.south)
    node[below right]{$\cdots$};
  \draw[arrowright={0.5}{}]
    ([xshift=+7pt, yshift=-7pt]F.south) -- ([xshift=+7pt]F.south);
\end{tikzpicture} for $A \in \mathbf{V^L}$ as morphisms, where $pa^L(A)$ denotes the parents of $A$ in the graph $L$.

Further, let  restr$(Free_L)$ denote the category arising from $Free_L$ after restricting to those morphisms where every generating box appears at most once.

\end{definition}

For an explicit construction of freely generated Markov categories, we refer to \cite{Fritz_2023}. 

Now, we are able to define a causal model over a general Markov category:
\begin{definition}\label{def:strict Markov functor}
A \textbf{causal model} over a DAG $L=(\mathbf{V^L},E^L)$ is a Markov functor $F_L:{Free_L}\rightarrow M$, where $M$ is a Markov category. 
\end{definition}

Given a causal model $F_L:{Free_L}\rightarrow M$ and a morphism $A \rightarrow B$ in $Free_L$ we will denote its image under $F_L$ as $p^{F_L}(B|A)$ and call these images distributions. Further, for a set of nodes $A\subset \mathbf{V^L}$, we simply denote the tensor product in $Free_L$ of these nodes as $A$.

By a result of \cite{jacobs2019} (Proposition 3.1), we can identify Causal Bayesian Networks (CBNs) over a DAG $L$ with functors of the form $Free_L \rightarrow Stoch$.

Further, the morphisms in the restricted category restr($Free_L$) exactly correspond to all possible types of interventional distributions:

\begin{proposition}\label{prop:restrictedMarkov}
    Consider a CBN $F_L:Free_L \rightarrow Stoch$ over a DAG $L=(\mathbf{V^L},E^G)$. For $A,B \subset \mathbf{V^L}$, the interventional distribution
        \[p^{F_L}(B|do(A)) := \underset{U:=\mathbf{V^L}\setminus (A \cup B)}{\int} \prod_{C \in \mathbf{V^L}\setminus A} p^{F_L}(C|pa^L(C)) dU\]
has a unique morphism in restr($Free_L$) associated to it and there are no other string diagrams in restr($Free_L$).
\begin{proof}
In restr($Free_L$), there is only one way to stack the generating morphisms to obtain a morphism with signature $A \rightarrow B$.
\end{proof}
\end{proposition}

In light of \cref{prop:restrictedMarkov}, for a causal model $F_L:Free_L \rightarrow M$ over a general Markov category $M$, we can refer to the images of morphisms in restr($Free_L$) as \textbf{the interventional distributions of $F_L$}.

\begin{example}
    Consider a CBN over a DAG $L=$ 

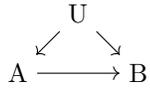
\begin{tikzpicture}[scale = 0.4]

\node (A) at (0,0) {A};
\node (U) at (2,2) {U};
\node (B) at (4,0) {B};

\draw[->] (U) -- (A);
\draw[->] (U) -- (B);
\draw[->] (A) -- (B);

\end{tikzpicture}   

with distribution $p(A B, U) = p(U) \cdot p(A|U) \cdot p(B|A,U)$. The distributions $p(A,B)$ and $p(B|do(A))$ correspond to the following string diagrams, respectively:

$
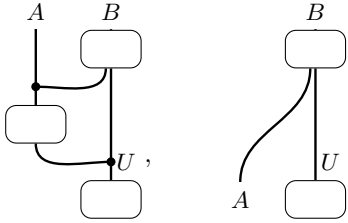
\begin{tikzpicture}[scale = 0.5, baseline,
  every node/.style={font=\small},
  wire/.style={line width=0.9pt},
  box/.style={draw,rounded corners,minimum width=8mm,minimum height=5mm,inner sep=2pt},
  dot/.style={circle,fill,inner sep=1.2pt}
]

  \node[box] (U) at (2,-1) {};
  \node[box] (AU) at (0,1) {};
  \node[box] (BAU) at (2,3) {};

  \node (A) at (0,4) {$A$};
  \node (B) at (2,4) {$B$};

  \coordinate (split1) at (2,0);
  \coordinate (split2) at (0,2);
  \node (Ul) at (2.4, 0) {$U$};

  \draw[wire] (U) -- (BAU.south);                          
  \draw[wire] (AU.north) -- (split2);                      
  \node[dot] at (split1) {};   
  \node[dot] at (split2) {}; 
  \draw[wire] (split1) to[out=-180,in=-90] (AU.south);
  \draw[wire] (split2) to[out=0,in=-90] (BAU.255);
  \draw[wire] (split2) -- (A);                   
  \draw[wire] (BAU.north) to (B);       

\end{tikzpicture},
\begin{tikzpicture}[scale=0.5, baseline,
  every node/.style={font=\small},
  wire/.style={line width=0.9pt},
  box/.style={draw,rounded corners,minimum width=8mm,minimum height=5mm,inner sep=2pt},
  dot/.style={circle,fill,inner sep=1.2pt}
]

  \node (Sp) at (-2,0) {};
  \node[box] (U) at (2,-1) {};
  \node[box] (BAU) at (2,3) {};

  \node (B) at (2,4) {$B$};

  \node (Aa) at (0,-1) {$A$};
  \node (Ul) at (2.4, 0) {$U$};

  \draw[wire] (U) -- (BAU.south);               
  \draw[wire] (Aa) to[out=90,in=-90] (BAU.255);
  \draw[wire] (BAU.north) to (B);       

\end{tikzpicture}$
\end{example}

\section{Categorically Unifying Causal Abstractions}\label{CausalAbstractions}
We now lay out our categorical framework for causal abstractions.

\subsection{Causal abstractions as deterministic natural transformations}

We give a categorical definition of causal abstractions:
\begin{definition}\label{def:abstractDefinition}
    A causal model $F_H:Free_H \rightarrow M$ is a \textbf{causal abstraction} of a causal model $F_L: Free_L \rightarrow M$ if there exists a Markov functor $\iota:Free_H \rightarrow Free_L$ that embeds restr($Free_H$) into restr($Free_L$) and if there exists a natural transformation $\tau: F_L\iota \Rightarrow F_H$ whose components are deterministic.
\end{definition}

To elaborate, we first discuss abstractions on the side of graphs:
\begin{definition}\label{def:operationsGraphicalAbstractions}
For a DAG $L=(\mathbf{V^L},E^L)$, define its \textbf{graphical abstractions} as the collection of DAGs that result from applying a sequence of the following two operations to $L$:
\begin{itemize}
    \item Deleting a node $A \in \mathbf{V^L}$ and adding edges from all of $A$'s parents to all its children. This operation is only valid if $A$ does not have two outgoing edges to different nodes, that is, if it is not a confounder. 
    \item Merging two nodes $A,B \in \mathbf{V^L}$ and combining all incoming edges to either $A$ or $B$ as incoming edges into the merged node and combining all outgoing edges of either $A$ or $B$ as outgoing edges of the merged node. 
\end{itemize}
\end{definition}

The following proposition describes the relation between a DAG $L$ and its graphical abstractions when mapped to the respective freely generated Markov categories.
\begin{proposition}\label{prop:embedding}
A DAG $H$ is a graphical abstraction of a DAG $L$ if and only if restr($Free_H$) embeds into restr($Free_L$).
\begin{proof}
    See appendix ~\ref{proofEmbedding}.
\end{proof}
\end{proposition}

\cref{prop:restrictedMarkov} and \cref{prop:embedding} together tell us that the graphical abstractions of a DAG $L$ are exactly those graphs whose induced types of interventional distributions can be related back to those induced by $L$. 

Consider causal models $F_H:Free_H\rightarrow M$, $F_L:Free_L \rightarrow M$ and a family of deterministic morphisms $(\tau_A:F_L(A) \rightarrow F_H(A))_{A \in \mathbf{V^H}}$. Since $Free_H$ is generated by the edges in $H$, $\tau$ already constitutes a natural transformation if for all $A \in \mathbf{V_H}$, the following diagram commutes:

\[
\begin{tikzcd}[column sep=huge, row sep=large] 
F_L(pa^H(A)) \arrow[r, "p^{F_L}(A|pa^H(A))"] \arrow[d, "\tau_{pa^H(A)}"'] & F_L(A) \arrow[d, "\tau_{A}"] \\
F_H(pa^H(A)) \arrow[r, "p^{F_H}(A|pa^H(A))"] & F_H(A)\
\end{tikzcd}
\]

In other words, $F_H$ is a causal abstraction of $F_L$ if the high-level mechanisms are compatible with the mechanisms associated to the corresponding clusters of low-level variables.

When restricting to $M=\text{Stoch}$, we can recover the following notion of causal abstraction between CBNs:

\begin{definition}\label{def:simpleD}
    Consider two CBNs $F_L, F_H$ over DAGs $L=(\mathbf{V^L},E_L),H=(\mathbf{V^H},E^H)$ with distribution $p^{F_L}, p^{F_H}$, respectively, such that the nodes of $H$ correspond to disjoint sets of nodes in $L$ and the associated domains are partitions of the domains in the associated cluster of variables. The causal model over $p^{F_H}$ is a causal abstraction between CBNs of $p^{F_L}$ if all interventional distributions coincide in the following sense: $\forall A,B \in \mathbf{V^H}, a \in F_L(A), b \in F_L(B):$
    \begin{align}\label{eq:simpleDef}
    & p^{F_L}(\tilde{b}|do(a)) = p^{F_H}(\tilde{b}|do(\tilde{a}))
    \end{align}

    To simplify notation, for a low-level value $a \in F_L(A)$ we denote its associated cluster of values as $\tilde{a}$, both when viewed on the high-level CBN, i.e. $\tilde{a} := \tau_A(a)\in F_H(A)$, as well as when viewed on the low-level CBN, i.e. $\tilde{a} :=\tau_A^{-1}(\tau_A(a)))\subset F_L(A)$.
    
\end{definition}

\cref{def:simpleD} states that it does not matter whether one first applies the causal mechanisms on the low-level model $F_L$ and then maps to the high-level model $F_H$ or vice-versa. We are now ready to present our main result proving that our categorical \cref{def:abstractDefinition} generalizes \cref{def:simpleD}. We prove that our result is a sufficient and necessary condition for unifying causal abstractions within a categorical framework. 
\begin{theorem}\label{thm:MainResult}
    Consider two CBNs $F_L,F_H$. Then $F_H$ is a causal abstraction of $F_L$  as in \cref{def:simpleD} \textbf{\underline{if and only if}} $F_H$ is a causal abstraction of $F_L$ as in the categorical \cref{def:abstractDefinition} in the case $M=\text{Stoch}$.
 \end{theorem}
 
\begin{proof}
    \textbf{\textit{Sufficiency}}: Let $F_H$ be a causal abstraction of $F_L$ as in the categorical \cref{def:abstractDefinition} given by a natural transformation $\tau: F_L \iota \Rightarrow F_H$. Any high-level node $A \in \mathbf{V^H}$ corresponds to the cluster $\iota(A)$ of low-level nodes.  Since $\tau$ respects monoidality, $\tau_{\mathbf{V^H}}$ factorizes as $\tau_{\mathbf{V^H}}=\underset{A \in \mathbf{V^H}}{\prod} \tau_{A}$. The natural transformation dictates that every distribution $p^{F_H}(A\in \mathbf{V^H})= p^{F_H}(A|I) = \eta_A \circ p^{F_L}(A|I)$ and hence every $\tau_A$ is surjective. Therefore, the maps $\tau_{A}$ give a cluster the domains of clusters of low-level nodes. Consider two high-level nodes $A, B$ and the interventional distribution $p(B|do(A))$. By \cref{prop:restrictedMarkov}, there exist unique associated string diagrams $A \rightarrow B$ in restr($Free_H$) and restr($Free_L$). Since $\tau$ is a natural transformation, the following diagram commutes:
      \[
\begin{tikzcd}[column sep=huge, row sep=large]
F_L(A) \arrow[r, "p^{F_L}(B|A)"] \arrow[d, "\tau_{A}"'] & F_L(B) \arrow[d, "\tau_{B}"] \\
F_H(A) \arrow[r, "p^{F_H}(B|A)"] & F_H(B)\
\end{tikzcd}
\]

Then 
\begin{align*}
p^{F_L}(\tilde{b}|a) &= \int_{b \in \tau_{B}^{-1}(\tilde{b})}p^{F_L}(b|a)db \\
                              &= \int_{b \in \tau_{B}^{-1}(\tilde{b})} p^{F_L}(b|a) \cdot \tau_{B}(\tilde{b}|b) db \\
                              &=\tau_{A}(\tilde{a}|a) \cdot p^{F_H}(\tilde{b}|\tilde{a}) \label{commutingUsed} \\
                              &= p^{F_H}(\tilde{b}|\tilde{a})
\end{align*}

\textbf{\textit{Necessity}}: Let $F_H$ be a causal abstraction between CBNs of $F_L$  as in \cref{def:simpleD}. Let $\iota$ map the objects of $Free_H$ to their counterparts in $Free_L$ induced by the clustering of low-level nodes. Since restr($Free_H$) and restr($Free_L$) exactly correspond to the respective types of interventional distribution by \cref{prop:restrictedMarkov}$, \iota$ has to map every morphism in restr($Free_H$) to its unique counterpart. Now assume this would not constitute an embedding; then $\iota$ is not functorial and hence there would be a type of interventional distribution in restr($Free_H$) that has no image in $Free_L$, violating \cref{eq:simpleDef}.

\end{proof}

\subsection{The category of causal abstractions}

It follows from \cref{def:abstractDefinition} that causal abstractions are compositional, i.e. if $F'$ is an abstraction of $F$ and $F''$ is an abstraction of $F'$, then $F''$ is an abstraction of $F$. We can therefore refer to a category of causal abstractions, denoted $C_M$ (where $M$ is a Markov category), that has causal models as objects and causal abstractions as morphisms. 
Consider a causal model $F \in C_M$; potential categories of interest are the slice category ($C_M/F$) that has as objects all causal models implementing $F$ and the coslice category ($F/C_M$) that has as objects all submodels of $F$.

\subsection{Non-aligned interventionals}\label{lax-monoidal}
In \cref{def:strict Markov functor} of a causal model $F:Free_L \rightarrow M$ , we have required the monoidal functor to be strict, i.e. for all variables $A,B \in L$, we require $F(A)\times F(B) = F(A\otimes B)$. In this case, all interventions are just products of single-variable interventions. However, in practical situations one may not always have such an alignment; consider the following example:

\begin{example}\label{example:strict monoidal}
Consider a deterministic causal model $F\in C_{\text{Stoch}}$ consisting of nodes $A,B,Y$ with edges $A\rightarrow Y, B\rightarrow Y$ and let $F(A)=F(B) = \mathbb{R}, F(A \otimes B) = \mathbb{R}^2$. Viewing $F(A),F(B), F(A\otimes B)$ as vector spaces, assume $F(A\otimes B)$ is the direct sum of two one-dimensional vector spaces representing two concepts of interest on which one wants to be able to intervene on, associated to nodes $A,B$, respectively. If these concepts are axis aligned, then we can simply model this causal model as a strict Markov functor, i.e. $F(A) \times F(B) = F(A \otimes B)$. However, if these two concepts are associated to orthogonal but not axis aligned subspaces, the coherence map $F(A)\times F(B)\rightarrow F(A\otimes B)$ is still an isomorphism but not the identity.
\end{example}
Hence, if the interventions of individual concepts are not just the product of the interventions on both concepts but still isomorphic to the product, we can model this by relaxing the assumption of a strict monoidal functor to a strong monoidal functor.

There may still be situations where this may be too strict; consider the following adaption of \cref{example:strict monoidal}:
\begin{example}\label{example:lax}
    We extend \cref{example:strict monoidal} by introducing a fourth node $C$ with outgoing edge $C \rightarrow Y$ and $F(A),F(B),F(C)=\mathbb{R},F(A\otimes B \otimes C) = \mathbb{R}^2$. Viewing these sets as vector spaces, the three 1-dimensional spaces $F(A),F(B),F(C)$ may correspond to three concepts on which one would like to be able to intervene, encoded in three linear directions in $\mathbb{R}^2$. These three linear directions cannot be orthogonal anymore. Modeling this as a Markov functor, the induced coherence map $F(A)\times F(B)\times F(C) \rightarrow F(A\otimes B \otimes C)$ is no longer an isomorphism.
\end{example}
Hence, in the case where the product of concepts one would like to be able to intervene on is larger (or smaller) than the domain of the affected set of variables, we have to further relax the assumption of a strong monoidal functor to a lax monoidal functor.

\subsection{Effect-focused causal abstractions}
Consider a CBN over a DAG $L: A\rightarrow B \rightarrow C$. The distribution factorizes as
\[ p(a,b,c) = p(c|b) \cdot p(b|a) \cdot p(a)\]
Given is a coarsening of the domains of the random variables, i.e. maps $\tau_A,\tau_B,\tau_C$ mapping from the respective domains to clustered domains. We want to know in which cases the resulting distribution on the coarsened domains still factorizes as a CBN over $L$, i.e. whether
\begin{align}\label{eq:factorization}
p(\tilde{a},\tilde{b}, \tilde{c}) = p(\tilde{c}|\tilde{b}) \cdot p(\tilde{b}|\tilde{a}) \cdot p(\tilde{a})
\end{align}
\cref{eq:factorization} is equivalent to
\begin{align*}
&\int_{b\in \tau_B^{-1}(\tilde{b})} p(\tilde{c}|b) \cdot p(b|\tilde{a}) \cdot p(\tilde{a}) db = \\
&\int_{b\in \tau_B^{-1}(\tilde{b})} p(\tilde{c}|b) \cdot p(b|\tilde{b}) db \cdot \int_{b\in \tau_B^{-1}(\tilde{b})} p(b|\tilde{a}) \cdot p(\tilde{a})db 
\end{align*}

Generally (without demanding a dependence of the two mechanisms on each other), this is only the case if either $p(\tilde{b}|\tilde{a})$ factorizes as 
\[p(\tilde{b}|\tilde{a}) = p(b|\tilde{b}) \cdot p(\tilde{b}|\tilde{a})\]
or if $p(\tilde{c}|b)$ is constant over the cluster $\tau_B^{-1}(\tilde{b})$. In the former case, the domain of variable $B$ is partitioned respecting the shared effect of parent variable $A$, whereas in the latter case it is partitioned respecting the shared effect on child variable $C$. The latter is captured by \cref{def:abstractDefinition}. The example above motivates the term \textit{effect-based} causal abstraction.

In the following, we will show how the former can be captured by defining the natural transformation in \cref{def:abstractDefinition} in the reverse direction.
First note that when we restrict to $M$=Stoch, the deterministic morphisms constituting the natural transformation have right inverses $\epsilon$:
\begin{lemma}\label{retract}
   Consider a CBN $F_H\in C_{\text{Stoch}}$ that is a causal abstraction of a CBN $F_L\in C_{\text{Stoch}}$ witnessed by a deterministic natural transformation $\tau: F_L\iota \rightarrow F_H$.
   The component morphisms $\tau_{A}$ have right inverses $(\epsilon_{A}:F_H(A) \rightarrow F_L (\iota(A))$.
   \begin{proof}
Define for every node $A$ in $\mathbf{V^H}$: 
\[\forall a \in F_L(A): \epsilon_A(a| \tilde{a}) := p^{F_L}(a|\tilde{a}) \]
Then \[\forall \tilde{a}' \in F_H(A):((\tau_A \circ \epsilon_A) (\tilde{a}')|\tilde{a}) = \mathbf{1}\{\tilde{a}'=\tilde{a}\}\]
   \end{proof}
\end{lemma}

We now let the morphisms $\epsilon$ constitute the natural transformation instead of $\tau$:

\begin{definition}\label{def:effect-based} 
    A causal model $F_H:Free_H \rightarrow \text{Stoch}$ is a \textbf{causal abstraction} of a causal model $F_L: Free_L \rightarrow \text{Stoch}$ if there exists a Markov functor $\iota:Free_H \rightarrow Free_L$ that embeds restr($Free_H$) into restr($Free_L$) and if there exists a natural transformation $\epsilon: F_H \Rightarrow F_L \iota$ whose components have deterministic left-inverses.
\end{definition}

We now discuss the implications of this definition, which will clarify the term  effect-focused causal abstraction:

Let causal model $F_H:Free_H \rightarrow M$ be an effect-focused causal abstraction of causal model $F_L: Free_L \rightarrow M$. Then the following diagram commutes for every high-level node $A \in \mathbf{V^H}$:

\begin{figure}[h]
\begin{tikzcd}[column sep=huge, row sep=large]
I \arrow[r, "p^{F_H}(A)"] \arrow[d, "\tau_I=id"'] & F_H(A) \arrow[d, "\epsilon_{A}"] \\
I \arrow[r, "p^{F_L}(A)"] & F_L(A)\
\end{tikzcd}
\label{fig:CommutingDiagramStrong}
\end{figure}

Then for every $a \in F_L(A)$:
\begin{align}
    p^{F_L}(a) &= \int_ {\Tilde{a}' \in F_H(A)} \epsilon_{A}(a|\tilde{a}') \cdot p^{F_H}(\tilde{a}')d\tilde{a}' \\
           &= \epsilon_{A}(a|\tilde{a}) \cdot p^{F_L}(\tilde{a})
\end{align}

If $p^{F_L}(\tilde{a})>0$, this implies
\begin{align}\label{eq:transition}
    \epsilon_{A}(a|\tilde{a}) =p^{F_L}(a|\tilde{a})
\end{align}

In other words, the transition probability from a cluster to its elements is just the probability conditioned on the cluster.

\begin{proposition}\label{prop:effect-focused}
      Let a CBN $F_H$ be an effect-focused abstraction of a causal model $F_L$ and consider a high-level node $A \in \mathbf{V^H}$, low-level values $b \in F_L(pa^H(A)), a \in F_L(A)$ such that $p^{F_L}(a),p^{F_L}(b)>0$. Then

\[p^{F_L}(a| \tilde{b}) = \epsilon_{A}(a|\tilde{a}) \cdot p^{F_L}(\tilde{a}|\tilde{b}) \]
\begin{proof}
    See appendix~\ref{proof_prop:effect-focused}.
\end{proof}
\end{proposition}

In other words, the map $\tau_{A}$ has to be a sufficient statistic for the distribution of $p^{F_L}(A|pa^H(A))$ parametrized by the clustered values of $F_H(pa^H(A))$.
Whereas \cref{def:simpleD} constrains $\tau$ to only cluster values that have the same effect on children variables, \cref{def:effect-based} constrains $\tau$ to only cluster values that are affected the same by parent variables. 

\paragraph{Example:}
We now discuss an example, adapted from \cite{beckers2019abstractingcausalmodels}, where this alternative definition may be useful.
Consider a voting scenario with 100 voters who can
either vote for or against a proposition.  The campaign for the
proposition can air some
subset of two advertisements to try to influence how the voters vote.   
The low-level model is
characterized by binary variables $A_1$,
$A_2$, and $X_i$, $i=1,\ldots,100$, . $X_i$ denotes voter $i$'s vote,
so $X_i=1$ if voter $i$ votes for the
proposition, and $X_i=0$ if voter $i$ votes against.  $A_i$ denotes
whether ad $i$ is run.

One would like to cluster the voters into groups that are equally affected by the ads. One way to do so, discussed in \cite{beckers2019abstractingcausalmodels},  is to cluster voters into groups for which the probability $p(x_i| a_1, a_2)$ coincides. However, we can further coarsen the partition of the voters by noting that two voters may have an initial bias towards the proposition independent of effect the ad (the cause) has on them, i.e. one may instead cluster voters into groups $\mathbf{X_c} = \{ X_{c_1},...,X_{c_{|X_c|}}\} \subset \mathbf{X}$ such that $\forall x_{c} :=(x_{c_1},...,x_{c_{|\mathbf{X_C}|}}) \in \mathbf{X_c}$:
\[p(x_{c}|a_1,a_2) =  p(x_{c}|\sum_{j=1}^{|\mathbf{X_c}|}x_{c_j}) \cdot p(\sum_{j=1}^{|\mathbf{X_c}|} x_{c_j}|a_1,a_2) \]

The causal mechanism can be captured without loss in a high-level model that does not model every voter but only groups of voters who are affected equally by the ads; within such a group of voters, only the sum of votes needs to be captured in the high-level domain.
This is the type of causal abstraction captured by \cref{def:effect-based}.

\section{Unifying Prior Perspectives}\label{Characterization}
We now demonstrate how our framework relates to and unifies several existing works on causal abstractions.

\subsection{Strong $\tau$ abstractions and $\tau$ constructive abstractions}

In their treatment of causal abstractions, \cite{beckers2019abstractingcausalmodels} differentiate between strong $\tau$-abstractions and constructive $\tau$-abstractions. A constructive $\tau$ abstraction can be seen as a deterministic causal abstraction in our framework with an additional context variable. Let a deterministic causal model $F_H \in C_{\text{Set}}$ be a causal abstraction (\cref{def:abstractDefinition}) of a deterministic causal model $F_L \in C_{\text{Set}}$ witnessed by a natural transformation $\tau$. Since $\tau$ is a natural transformation between strict monoidal functors, it has to preserve the monoidal structure, i.e. 
\begin{align}\label{eq:factorizationTau}
    \forall A,B \in \mathbf{V^H}: \tau_{A} \times \tau_{B} = \tau_{A \otimes B}
\end{align}
In other words, the map $\tau: F_L(\mathbf{V^H}) \rightarrow F_H(\mathbf{V_H})$ factorizes as $\tau = (\tau_1,...,\tau_{|V_H|})$. \cite{beckers2019abstractingcausalmodels} call such maps constructive.
A strong $\tau$-abstraction is a relaxation of a constructive $\tau$-abstraction where there may not be an alignment between clusters of low-level variables and high-level nodes, i.e. $\tau$ may not factorize as in \cref{eq:factorizationTau}. A special case of a strong $\tau$-abstraction are the constructions given in section \cref{lax-monoidal}, i.e. when $F_L$ may only be a strong or even lax monoidal functor. In \cite{beckers2019abstractingcausalmodels}, the authors conjecture that under a few minor technical conditions, every strong $\tau$-construction is also a constructive $\tau$-abstraction. The examples of \cref{lax-monoidal} serve as an example of strong $\tau$-abstractions that are not constructive $\tau$-abstractions
\subsection{Causal abstractions in mechanistic interpretability}

\cite{geiger2025causalabstractiontheoreticalfoundation} unify several concepts of mechanistic interpretability methods in the language of causal abstractions. Their definition of constructive abstractions coincides with our concept of a deterministic causal abstraction.

Our discussion of non-aligned interventional sets in section \cref{lax-monoidal} relates to the concepts of the linear representation hypothesis and superposition in the context of mechanistic interpretability. The linear representation hypothesis postulates that concepts in the activation space of neural networks are encoded as linear directions; superposition implies that these linear subspaces are not simply the orthogonal axes induced by the neurons. Hence, intervening on a single concept is not possible by just intervening on a single neuron; this is captured by examples of type \cref{example:strict monoidal}.

Superposition further implies that the number of concepts can be larger than the number of dimensions, and hence the concepts cannot be stored in orthogonal directions. The goal of sparse autoencoders can then be viewed as the goal of finding a deterministic causal abstraction $F_{\text{sparse}}\in M_{Set}$ of a neural network $F$ such that $F_{\text{sparse}}$ is a strict monoidal functor, while the neural network $F$ is only a lax monoidal functor; analogous to \cref{example:lax}.

In mechanistic interpretability one is further interested in whether a network implements a certain algorithm or task.  View a neural network as a deterministic causal model $F\in C_{Set}$ with parentless nodes corresponding to the input and childless nodes corresponding to the output. Then the objects in the coslice category $F/C_{Set}$ are the subnetworks of $F$. On the other hand, for some algorithm $A \in C_{Set}$, $F$ implements the algorithm $A$ if $F$ is an object in the slice category $C_{Set}/A$. Assume one is interested in how the network $F$ performs a certain task given by a set of input-output pairs. This set of input-output pairs can be seen as a two-node causal model: $A:X_{\text{input}}\rightarrow X_{\text{output}}$ and is causal abstraction of the neural network $F$. One is then interested in those submodels of the network that already implement the task, i.e. in the morphisms $F\rightarrow F'$ in $C/A$. In mechanistic interpretability, these morphisms are the objects of interest when finding subcircuits.

\subsection{Cluster-DAGs}
\cite{CDAGS} consider abstractions between causal models with unobserved confounders, encoded in acyclic directed mixed graphs (AMDGs). ADGMs can have bidirected edges representing unobserved confounders.

\begin{definition}\label{def:ADMGs}
    A causal model $F_{L'}\in C_M$ over a DAG $L'=(\mathbf{V^{L'}},E^{L'})$ is a \textbf{causal model over an ADMG} $L=(\mathbf{V^L},E^L)$ if the nodes of $L'$  can be divided $\mathbf{V^{L'}}=\mathbf{V^L} \sqcup \mathbf{U}$ into endogeneous nodes $\mathbf{V^L}$ and exogeneous nodes $\mathbf{U}$ such that $L$ is the latent projection of $L'$.
    
     An $ADMG$ $H$ is a \textbf{graphical abstraction between ADMGs} of ADMG $L$ if there are DAGs $L',H'$ whose latent projections are $L, H$, respectively, and such that $H'$ is a graphical abstraction of $L'$.  
\end{definition}

\begin{lemma}\label{lemma:factorization over ADMGs}
     Let $F_{L'}\in C_M$ be a causal model over an ADMG $L=(\mathbf{V^L},E^L)$ and let $H=(\mathbf{V^H},E^H)$ be a graphical abstraction between ADMGs of $L$. Then restricted to a subcategory of $Free_{L'}$, $F_{L'}$ is a causal model over the ADMG $H$.
    
\begin{proof}
     Consider the DAGs $L'',H''$ associated with the graphical abstraction between ADGMs $L,H$; w.l.o.g. we can choose $L''=L'$ and $H''$ such that there exists a bidirected edge $A \leftrightarrow B$ in $H$ iff there exists a node $U$ with outgoing edges $U\rightarrow A, B$ in $H''$. Since $H''$ is a graphical abstraction of $L''$, by \cref{prop:embedding} there exists a functor $\iota$ such that $F_{L'} \circ \iota:Free_{H''}\rightarrow M$ is a causal abstraction of $F_{L'}$. 
\end{proof}
\end{lemma}
\cref{lemma:factorization over ADMGs} and \cref{thm:MainResult} lead to the following corollary, which subsumes Theorem 2 and Theorem 5 in \cite{CDAGS}, who proved the following corollary in the case of clustering low-level variables which is a special case of graphical abstractions as we defined them (\cref{def:abstractDefinition}).

\begin{corollary}
    Consider a CBN with distribution $p$ that factorizes over the ADMG $L$, and let $H$ be a graphical abstraction between ADMGs of $L$. Then restricted to high-level nodes $\mathbf{V^H}$, all interventional distributions factorize over the ADMG $H$, i.e. for all $A \in \mathbf{V^H}$:
     \begin{align*} & p(\mathbf{V^H}\setminus\{A\}| do(A)) =  \\
                  &                  \int_U p(U) \cdot \left (\prod_{C \in \mathbf{V^H} \setminus \{A\}} p(C| pa_{H}(C),U_C) \right)dU
    \end{align*}
    such that $U_{C}\cap U_{C'} \neq \emptyset$ if and only if there is a bidirected edge $C \leftrightarrow C'$ in $H$.
\end{corollary}

Given a CBN over an ADMG $L=(\mathbf{V^L},E^L)$, \cite{CDAGS} further show how applying the rules of Pearl's do-calculus on the high-level graph induced by a partition over $V$ gives valid results on the low-level graph $L$. We prove a generalized statement for all graphical abstractions between ADMGs:

\begin{theorem}\label{thm:do-calculus}

Consider a CBN $F$ over an ADMG $L=(\mathbf{V^L},E^L)$ and let $H=(\mathbf{V^H}, E^H)$ be a graphical abstraction between ADMGs of $L$.  Then applying Pearl's do-calculus on $H$ gives valid results on the low-level graph $L$, i.e.
for any disjoint subsets of clusters ${X}, {Y}, \*{Z}, \*{W} \subseteq \mathbf{V^H}$, the following three rules hold: 
$$
\begin{aligned}
    &\textbf{Rule 1:} \ p^F(Y|do(X), Z, W) = p^F(Y | do(X), W) \\
    & 
     \qquad \qquad \qquad
     \text{if } (\*{Y}\perp \! \!\! \perp \*{Z} | \*{X}, \*{W})_{H_{{\overline{\*X}}}}  \\ 
\end{aligned}
$$

$$
\begin{aligned}
    & \textbf{Rule 2:} \ p^F(Y | do(X), do(Z), W) = p^F(Y | do(X), Z, W) \\
     & 
     \qquad \qquad \qquad
     \text{if } (Y \perp \! \!\! \perp \*{Z} | \*{X}, \*{W})_{H_{{\overline{\*{X}} \underline{\*{Z}}}}} \\
    & \textbf{Rule 3:} \ p^F(Y | do(X), do(Z), W) = p^F(Y | do(X), W) \\
    & 
    \qquad \qquad \qquad
    \text{if } (\*Y \perp \! \!\! \perp \*Z | \*X, \*W)_{H_{{\overline{\*X}\overline{\*Z_H(\*W)}}}}\\
\end{aligned}
$$
where $H_{{\overline{\*X}\underline{\*Z}}}$ is obtained from $H$ by removing the edges 
into $\*X$ and out of $\*Z$,
and $\*Z_H(\*W)$ is the set of nodes in $Z$ that are non-ancestors of any node of $\*W$ in $H$. 
\begin{proof}
See appendix~\ref{proof_do_calculus}.
\end{proof}
\end{theorem}

\begin{remark}
    This generalizes Theorem 3 by \cite{CDAGS}, as they prove the statement for partitions of the low-level variables which is subsumed by our definition of graphical abstractions (\cref{def:operationsGraphicalAbstractions}).
\end{remark}

\subsection{$\Phi$-Abstractions}
The work by \cite{otsuka2022} is the closest to our proposed framework. Given two CBNs $F_L, F_H$ - viewed as elements in $Stoch^{Free_L}, Stoch^{Free_H}$ - they call $F_H$ a \textbf{$\phi$-abstraction} of $F_L$ if there exists a graph homomorphism $\phi: L\rightarrow H$ such that there exists a natural transformation $\alpha: F_L \Rightarrow F_H \Phi$, where $\Phi$ is the functor $\Phi : Free_L \rightarrow Free_H$ induced by $\phi$. 
One issue with this definition is that node clusterings cannot be straightforwardly modeled as $\Phi$-abstractions:

\begin{example}\label{ex:ABCgraph}
    Consider a CBN $F_L\in C_{\text{Stoch}}$ over a DAG $L:A \rightarrow B \rightarrow C$ and let $F_H$ be the CBN over a single node $ABC$ with the same distribution as the joint distribution of $F_L$, i.e. for $(a,b,c) \in F_L(A\otimes B\otimes C)(=F_H(ABC)): p^{F_L}(a,b,c) = p^{F_H}(a,b,c)$. While there exists a unique graph homomorphism $L \rightarrow H$, there is no straightforward way to define a natural transformation $F_L \rightarrow F_H \Phi$, as this would require to define deterministic maps $\tau_A:F_L(A)\rightarrow F_H(A \otimes B\otimes C), \tau_B:F_L(B)\rightarrow F_H(A \otimes B \otimes C), \tau_C:F_L(C)\rightarrow F_H(A \otimes B \otimes C)$ such that $\tau_A =  \tau_B \circ p^{F_L}(B|A)$, $\tau_B =  \tau_C \circ p^{F_L}(C|B)$
\end{example}

Further, the requirement of a graph homomorphism may be too restrictive. Consider the same CBN as in \cref{ex:ABCgraph} and let $F_H$ be the CBN over graph $H:A \rightarrow C$ induced by $F_L$, i.e. $p^{F_H}(A) = p^{F_L}(A),p^{F_H}(C|A) = p^{F_L}(C|A)$. Since there is no corresponding graph homomorphism $L \rightarrow H$, this cannot be framed as a $\phi$-abstraction.
In comparison, we restricted to graphical abstraction, which by \cref{prop:embedding} together with \cref{prop:restrictedMarkov} is the most general set of graphs that are consistent under interventions.

In addition, their framework does not recover existing notions of causal abstractions, as there is no equivalent of \cref{thm:MainResult}.

\subsection{Neural Causal Abstractions}
\cite{xia2024neuralcausalabstractions} discuss causal abstractions of probabilistic causal models over the three layers of the Pearl Causal Hierarchy (PCH). Our categorical \cref{def:abstractDefinition} is more general in nature on the interventional layer, as their definition of $\tau$-consistency coincides with our definition in the restricted case of $M=\text{Stoch}$.

They further differentiate between intervariable clustering, i.e. clustering of nodes, and intravariable clustering, i.e. clustering of individual variable domains. In our framework of \cref{def:abstractDefinition}, intravariable clustering is encoded in the morphisms defining the natural transformation, whereas intervariable clustering is encoded in the functor $\iota$.

\section{Conclusion}

We have introduced a novel categorical framework that can recover several useful notions of causal abstractions. We prove that our result is a sufficient and necessary condition for unifying causal abstractions within a categorical framework. We have shown the effectiveness of this definition by being able to give concise string diagrammatic proofs of existing results. We also theoretically show how several previous works can be easily encapsulated by our proposed framework.

One interesting future direction is to explore situations where allowed interventions do not coincide with individual variable domains and abstractions between such causal models in more depth. For example, one question would be whether the \textit{exact transformations} introduced in \cite{rubenstein2017causalconsistencystructuralequation} can fit into our categorical framework. Further, one may try to prove \cref{thm:do-calculus} for general Markov categories. Including cyclic causal models within our framework is also an interesting future direction.

\bibliographystyle{plainnat}
\bibliography{Biblio2}

\clearpage
\appendix
\thispagestyle{empty}

\onecolumn
\aistatstitle{Supplementary Materials}





\section{MISSING PROOFS}


\subsection{Proof of \cref{prop:embedding}}\label{proofEmbedding}
\begin{proof}
Sufficiency: By \cref{def:operationsGraphicalAbstractions}, the statement follows after proving it for the cases that $H$ arises after either one of the two operations in \cref{def:operationsGraphicalAbstractions}. In both cases, $\iota$ is defined by mapping the generators of $Free_H$ to their natural counterparts in $Free_L$ following the respective operations on the graph:
\begin{itemize}
    \item Assume $H$ results from merging two nodes $A, B$: To construct $\iota$, it suffices to map the generators of $Free_H$ to $Free_L$. On the side of generating objects (i.e. the nodes), let $\iota$ map the merged node $(A,B)$ to $A\otimes B$, and let $\iota$ be the identity on all other nodes. For all morphisms except those  where $(A,B)$ appears in the incoming or outgoing wires, let $\iota$ be the identity. The remaining boxes have unique counterparts in the string diagram in $Free_L$ describing the full factorization. Let $\iota$ map the boxes to these counterparts, e.g. let 
$\iota\left(\begin{tikzpicture}[baseline,
  every node/.style={font=\small},
  wire/.style={line width=0.9pt},
  box/.style={draw,rounded corners,minimum width=8mm,minimum height=5mm,inner sep=2pt},
  dot/.style={circle,fill,inner sep=1.2pt}
]

  \node (A) at (0.5,-1) {$pa^(A\otimes B)$};

  \node[box] (f) at (0.5, 0) {};

  \node (Cout) at (0.5,1) {$A\otimes B$};

  \draw[wire] (A) -- (f.south);                          
  \draw[wire] (f.north) -- (Cout.south);                     

  \node[left=1mm of f] {};

\end{tikzpicture}\right)= 
\begin{tikzpicture}[baseline,
  every node/.style={font=\small},
  wire/.style={line width=0.9pt},
  box/.style={draw,rounded corners,minimum width=8mm,minimum height=5mm,inner sep=2pt},
  dot/.style={circle,fill,inner sep=1.2pt}
]

  \node (A) at (0,-1) {$X$};
  \node (B) at (2,-1) {$Y$};
  \node (cap) at (1,-1) {$Z$};

  \node[box] (f) at (0,0.1) {};
  \node[box] (g) at (2,1.3) {};

  \node (Cout) at (0,2) {$A$};
  \node (Dout) at (2,2) {$B$};

  \coordinate (split) at (0,0.6);
  \coordinate (split2) at (1,-0.4);

  \draw[wire] (A) -- (f.south);                          
  \draw[wire] (f.north) -- (split);                      
  \draw[wire] (cap) -- (split2);   
  \node[dot] at (split) {};   
  \node[dot] at (split2) {}; 
  \draw[wire] (split2) to[out=0,in=-90] (f.south);
  \draw[wire] (split2) to[out=0,in=-90] (g.260);
  \draw[wire] (split) -- (Cout.south);                   
  \draw[wire] (split) to[out=4,in=-90] (g.250);       
  \draw[wire] (B) to[out=90,in=-90] (g.south);      
  \draw[wire] (g.north) -- (Dout.south);                 

\end{tikzpicture}$

where $X:=pa^L(A)\setminus (pa^L(A) \cap pa^L(B)), Y:= pa^L(B) \setminus ( (pa^L(A) \cap pa^L(B)) \cup A), Z:=(pa^L(A) \cap pa^L(B))$ and w.l.o.g. $A\leq B$ in the topological order on the nodes $\mathbf{V^L}$ (we can choose any topological order). The wire from $A$ to $B$ only exists if $A$ is a parent of $B$ in $L$.

   Since this is the only way to stack the generating morphisms in $Free_L$ without using any of them more than once to obtain a morphism of the same signature, the full subcategory of restr($Free_L$) without objects $A, B$ is restr($Free_H)$ and hence, restr($Free_H$) embeds into restr($Free_L$).

    \item Assume $H$ results after removing a node $A$: Again, to construct $\iota$, it suffices to map the generators of $Free_H$ to $Free_L$. On generating objects, this is the identity. On all generating morphisms/ boxes coming from edges untouched by the removing operation, it is also the identity. As before, the remaining morphisms/ boxes have unique counterparts in the string diagram in $Free_L$ describing the full factorization.  Since this is the only way to stack the generating morphisms in $Free_L$ without using any of them more than once to obtain a morphism of the same signature, the full subcategory of restr($Free_L$) without objects that are tensor products including $X$ is restr($Free_H$).

\end{itemize}

Necessity: We first show that if $H$ is the graph that results after deleting a confounder in $L$, then rest($Free_H$) does not embed into restr($Free_L$). In this case there exist three nodes $A,B,U$ such that there are paths $A \rightarrow A, U\rightarrow B$ in $L$. If there is no directed path $A \rightarrow B$ or $B \rightarrow A$, then the string diagram of the (unique) morphsim $I \rightarrow A \otimes B$ is disconnected in $Free_H$ but connected in $Free_L$ and hence restr($Free_H$) does not embed into restr($Free_L$). Now assume there is a directed path between $A$ and $B$, w.l.o.g. $A \rightarrow B$. Assume an embedding $\iota$ would exist. The unique string diagrams with signature $A\rightarrow B$ and $I \rightarrow B$ in restr($Free_H$) are of the form   
$
\begin{tikzpicture}[baseline,
  every node/.style={font=\small},
  wire/.style={line width=0.9pt},
  box/.style={draw,rounded corners,minimum width=8mm,minimum height=5mm,inner sep=2pt},
  dot/.style={circle,fill,inner sep=1.2pt}
]

  \node (A) at (0.5,-1) {$A$};

  \node[box] (f) at (0.5, 0) {};

  \node (Cout) at (0.5,1) {$B$};

  \draw[wire] (A) -- (f.south);                          
  \draw[wire] (f.north) -- (Cout.south);                     

  \node[left=1mm of f] {};

\end{tikzpicture},
\begin{tikzpicture}[baseline,
  every node/.style={font=\small},
  wire/.style={line width=0.9pt},
  box/.style={draw,rounded corners,minimum width=8mm,minimum height=5mm,inner sep=2pt},
  dot/.style={circle,fill,inner sep=1.2pt}
]
  \node[box] (A) at (0.5,-1){};
  \node[box] (f) at (0.5, 0) {};
  \node (X) at (0,0) {};

  \node (Cout) at (0.5,1) {$B$};
  \draw[wire] (A) -- (f.south) node[midway, right] {$A$};                          
  \draw[wire] (f.north) -- (Cout.south);                     
  \node[left=1mm of f] {};
\end{tikzpicture}
$

respectively, whereas the unique string diagrams in restr($Free_L$) of the same signature are 
$
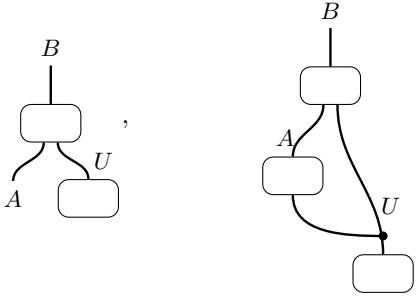
\begin{tikzpicture}[baseline,
  every node/.style={font=\small},
  wire/.style={line width=0.9pt},
  box/.style={draw,rounded corners,minimum width=8mm,minimum height=5mm,inner sep=2pt},
  dot/.style={circle,fill,inner sep=1.2pt}
]

  \node (A) at (0,-1) {$A$};
  \node[box] (cap) at (1,-1) {};

  \node[box] (f) at (0.5,0) {};

  \node (Cout) at (0.5,1) {$B$};
  \draw[wire] (A)  to[out=90, in=-90] (f.250);  
  \draw[wire] (cap.north) to[out=90, in=-90] (f.290);  
  \node (UU) at (1.2,-0.5) {$U$};
  \draw[wire] (f.north) -- (Cout);    

\end{tikzpicture},
\begin{tikzpicture}[baseline,
  every node/.style={font=\small},
  wire/.style={line width=0.9pt},
  box/.style={draw,rounded corners,minimum width=8mm,minimum height=5mm,inner sep=2pt},
  dot/.style={circle,fill,inner sep=1.2pt}
]

  \node[box] (A) at (0,-0.7) {};
  \node[box] (cap) at (1.2,-2) {};
  
  \node (X) at (-2,0) {};
  \node[box] (f) at (0.5,0.5) {};

  \coordinate (split) at (1.2,-1.5);
  \node[dot] at (split) {};   

  \node (B) at (0.5,1.5) {$B$};
  \draw[wire] (A) to[out=90, in=-90] (f.250);  
  \draw[wire] (split) to[out=180, in=-90] (A.south);  
  \draw[wire] (cap.north) to[out=90, in=-90] (f.290);  
  \node (UU) at (1.3,-1.1) {$U$};
  \node (AA) at (-0.1,-0.2) {$A$};
  \draw[wire] (f.north) -- (B);    
\end{tikzpicture}
$

respectively. Hence restr($Free_H$) cannot be a full subcategory of restr($Free_L$). 

Back to the general case, assume $\iota:Free_H \rightarrow Free_L$ exhibits $Free_H$ as a full subcategory of $Free_L$. The nodes appearing in the tensor product that any node in $H$ is mapped to are clusters of nodes in $L$. The nodes of $L$ that do not appear in any of these images are the removed nodes. Since removal and merging operations commute, we can first perform all merging operations, then merge all nodes that will be removed, and finally remove of that node. Now by our previous considerations, this only leads to a full subcategory if the final node to be removed is not a confounder.

\end{proof}

\subsection{Proof of \cref{prop:effect-focused}}
\label{proof_prop:effect-focused}
\begin{proof}
The following diagram has to commute:  
  \[
\begin{tikzcd}[column sep=huge, row sep=large] \label{CommutingDiagram2}
F_L(pa^H(A)) \arrow[r, "p^{F_L}(A|pa^H(A))"] & F_L(A) \\
F_H(pa^H(A)) \arrow[r, "p^{F_H}(A|pa^H(A))"] \arrow[u, "\epsilon_{pa^H(A)}"] & F_H(A)  \arrow[u, "\epsilon_{A}"]\
\end{tikzcd}
\]
Therefore,
\begin{align}
p^{F_L}(a|\tilde{b}) &= \int_{b\in \tau^{-1}_{pa^H(A)}(\tilde{b})} p^{F_L}(a|b) \cdot p^{F_L}(b|\tilde{b})db \\
                     &= \int_{b\in \tau^{-1}_{pa^H(A)})\tilde{b}} p^{F_L}(a|b) \cdot \epsilon_{pa^H(A)}(b|\tilde{b}) db \label{eq:pos1}\\
                     &=  \epsilon_{A}(a| \tilde{a}) \cdot p^{F_H}(\tilde{a}|\tilde{b}) \label{eq:comm1}\\
                     &= p^{F_L}(a|\tilde{a}) \cdot p^{F_H}(\tilde{a}|\tilde{b}) \label{eq:pos2}\\
                     &=  p^{F_L}(a|\tilde{a}) \cdot \int_{a \in \tau_A^{-1}(\tilde{a})} p^{F_L}(a|\tilde{a}) \cdot p^{F_H}(\tilde{a}|\tilde{b})da \\
                     &=  p^{F_L}(a|\tilde{a}) \cdot \int_{a \in \tau_A^{-1}(\tilde{a})} \epsilon_A(a|\tilde{a}) \cdot p^{F_H}(\tilde{a}|\tilde{b})da \label{eq:pos3}\\
                     &= p^{F_L}(a|\tilde{a}) \cdot \int_{b \in \tau_{pa^H(A)}^{-1}(\tilde{b})}p^{F_L}(\tilde{a}|b) \cdot \epsilon_{pa^H(A)}(b|\tilde{b})db  \label{eq:comm2}\\
                     &= p^{F_L}(a|\tilde{a}) \cdot \int_{b \in \tau_{pa^H(A)}^{-1}(\tilde{b})}p^{F_L}(\tilde{a}|b) \cdot p^{F_L}(b|\tilde{b})db  \label{eq:pos4}\\
                     & = p^{F_L}(a|\tilde{a}) \cdot p^{F_L}(\tilde{a}|\tilde{b})  \\
\end{align}
where we used \cref{CommutingDiagram2} in \cref{eq:comm1} and \cref{eq:comm2} and where we used \cref{eq:transition} in \cref{eq:pos1}, \cref{eq:pos2}, \cref{eq:pos3}, and \cref{eq:pos4}.
\end{proof}

\subsection{Proof of \cref{thm:do-calculus}}
\label{proof_do_calculus}
\begin{proof}
Let $L', H'$ be the DAGs associated to $L,H$ according to \cref{def:ADMGs}, respectively, that include unobserved confounders. 
\begin{itemize}
    \item \textbf{Rule 1}: Let $ (\*{Y}\perp \! \!\! \perp \*{Z} | \*{X}, \*{W})_{H_{\overline{\*X}}}$. Then also 
     $ (Y\perp \! \!\! \perp \*{Z} | \*{X}, \*{W})_{H'_{\overline{\*X}}}$. Since $H'$ is a graphical abstraction of $L'$, $H'_{\overline{X}}$ is also a graphical abstraction of $L'_{\overline{X}}$ and there exists a monoidal functor $\iota: Free_{H'_{\overline{X}}} \rightarrow Free_{L'_{\overline{X}}}$. 
    By Proposition 28 in \cite{fritz2023dseparationcriterioncategoricalprobability}, in the string diagram associated to the factorization of $Y,Z,X,W$ in $Free_{H'_{\overline{X}}}$, the wires corresponding to $Y, Z$ are not connected after removing the wire corresponding to $X,W$. By monoidality of $\iota$, the same holds in $Free_{L'_{\overline{X}}}$. Applying again Proposition 28 in \cite{fritz2023dseparationcriterioncategoricalprobability}, $X,Z$ are d-separated given $X,W$ in $L'_{\overline{X}}$ and therefore also in $L_{\overline{X}}$. Then applying the first rule of do-calculus on $L$, the statement follows.
    \item \textbf{Rule 2}: The proof of rule 2 is analogous to the proof of rule 1.
    \item \textbf{Rule 3}: Let $(\*Y \perp \! \!\! \perp \*Z | \*X, \*W)_{H_{{\overline{\*X}\overline{\*Z_H(\*W)}}}}$. Then also  $(\*Y \perp \! \!\! \perp \*Z | \*X, \*W)_{H'_{{\overline{\*X}\overline{\*Z_H(\*W)}}}}$. Since $H'$ is a graphical abstraction of $L'$, $H'_{\overline{X}}$ is also a graphical abstraction of $L'_{\overline{X}}$ and there exists a monoidal functor $\iota: Free_{{H'_{{\overline{\*X}\overline{\*Z_H(\*W)}}}}} \rightarrow Free_{{L'_{{\overline{\*X}\overline{\*Z_H(\*W)}}}}}$. 
    By Proposition 28 in \cite{fritz2023dseparationcriterioncategoricalprobability}, in the string diagram associated to the factorization of $Y,Z,X,W$ in $Free_{{H'_{{\overline{\*X}\overline{\*Z_H(\*W)}}}}}$, the wires corresponding to $Y, Z$ are not connected after removing the wire corresponding to $X,W$. By monoidality of $\iota$, the same holds in $Free_{L'_{{\overline{\*X}\overline{\*Z_H(\*W)}}}}$. Since $Z_H(W) \subset Z_L(W)$ (which is easy to check), this is also true for $Free_{L'_{{\overline{\*X}\overline{\*Z_L(\*W)}}}}$. Applying again Proposition 28 in \cite{fritz2023dseparationcriterioncategoricalprobability}, $X,Z$ are d-separated given $X,W$ in $L'_{{\overline{\*X}\overline{\*Z_L(\*W)}}}$ and therefore also in $L_{{\overline{\*X}\overline{\*Z_L(\*W)}}}$. Then applying the third rule of do-calculus on $L$, the statement follows.
\end{itemize}
\end{proof}








\end{document}